\documentclass[runningheads]{llncs}
\usepackage{graphicx}
\usepackage{amssymb}
\usepackage{amsmath}

\usepackage{algorithm}
\usepackage[noend]{algpseudocode}

\begin{document}
\title{Reward Potentials for Planning with Learned Neural Network Transition Models}
\titlerunning{Reward Potentials for Planning with Learned Neural Networks}
%
\author{Buser Say\inst{1,2}\thanks{This work is done during author's visit to Australian National University.} \and
Scott Sanner\inst{1,2} \and
Sylvie Thi\'ebaux\inst{3}}
\authorrunning{Say et al.}
%
\institute{University of Toronto, Canada 
\email{\{bsay,ssanner\}@mie.utoronto.ca}\\
\and Vector Institute, Canada\\
\and Australian National University, Australia
\email{sylvie.thiebaux@anu.edu.au}}
\maketitle              
\begin{abstract}
Optimal planning with respect to learned neural network (NN) models in continuous action and 
state spaces using mixed-integer linear programming (MILP) is a challenging task 
for branch-and-bound solvers due to the poor linear relaxation of the underlying MILP 
model. For a given set of features, potential heuristics provide an efficient framework 
for computing bounds on cost (reward) functions. In this paper, we model the problem of 
finding optimal potential bounds for learned NN models as a bilevel program, and solve 
it using a novel finite-time constraint generation algorithm. We then strengthen the linear 
relaxation of the underlying MILP model by introducing constraints to bound the reward 
function based on the precomputed reward potentials. Experimentally, we show that our 
algorithm efficiently computes reward potentials for learned NN models, and that the overhead of 
computing reward potentials is justified by the overall strengthening of the underlying MILP 
model for the task of planning over long horizons.

\end{abstract}

\section{Introduction}

\noindent Neural networks (NNs) have significantly improved the ability of autonomous systems 
to learn and make decisions for complex tasks such as image recognition~\cite{Krizhevsky2012}, 
speech recognition~\cite{Deng2013}, and natural language processing~\cite{Collobert2011}. 
As a result of this success, formal methods based on representing the decision making problem 
with NNs as a mathematical programming model, such as verification of 
NNs~\cite{Katz2017,Narodytska2018} and optimal planning with respect to the learned 
NNs~\cite{Say2017} have been studied.

In the area of learning and planning, Hybrid Deep MILP Planning~\cite{Say2017} (HD-MILP-Plan) 
has introduced a two-stage data-driven framework that (i) learns transitions models with 
continuous action and state spaces using NNs, and (ii) plans optimally with respect to 
the learned NNs using a mixed-integer linear programming (MILP) model. It has been 
experimentally shown that optimal planning with respect to the learned NNs~\cite{Say2017} 
presents a challenging task for branch-and-bound (B\&B) solvers~\cite{IBM2019} due 
to the poor linear relaxation of the underlying MILP model that has a large number of 
\textit{big-M} constraints. 

In this paper, we focus on the important problem of improving the efficiency 
of MILP models for decision making with learned NNs. In order to tackle this challenging 
problem, we build on potential heuristics~\cite{Pommerening2015,Jendrik2015},
which provide an efficient framework for computing a lower bound on the cost of a 
given state as a function of its features. In this work, we describe the problem of 
finding optimal potential bounds for learned NN models with 
continuous inputs and outputs (i.e., continuous action and state spaces) as a bilevel program, 
and solve it using a novel finite-time constraint generation algorithm. Features of our linear 
potential heuristic are defined over the hidden units of the learned NN model, thus providing 
a rich and expressive candidate feature space. We use our constraint 
generation algorithm to compute the potential contribution (i.e., reward potential) of each hidden 
unit to the reward function of the HD-MILP-Plan problem. The precomputed reward potentials are 
then used to construct linear constraints that bound the reward function of HD-MILP-Plan, and 
provide a tighter linear relaxation for B\&B optimization by exploring smaller number of nodes 
in the search tree. Experimentally, we show 
that our constraint generation algorithm efficiently computes reward potentials for learned 
NNs, and that the overhead computation is justified by the overall strengthening of the 
underlying MILP model for the task of planning over long horizons. 

Overall this work 
bridges the gap between two seemingly distant literatures -- research on planning heuristics 
for discrete spaces and decision making with learned NN models in continuous action and state 
spaces. Specifically, we show that 
data-driven NN models for planning can benefit from advances in heuristics and from their impact on 
the efficiency of search in B\&B optimization.

\section{Preliminaries}

We review the HD-MILP-Plan framework for optimal planning~\cite{Say2017} with learned NN models, potential heuristics~\cite{Pommerening2015}  
as well as bilevel programming~\cite{Bard2000}.

\subsection{Deterministic Factored Planning Problem Definition}

A deterministic factored planning problem is a tuple $\Pi = \langle S,A,C,T,I,G,R \rangle$ 
where $S = \{s_1, \dots, s_{n}\}$ and $A = \{a_1, \dots, a_{m}\}$ are sets of state and 
action variables with continuous domains, 
$C: \mathbb{R}^{|S|} \times \mathbb{R}^{|A|} \rightarrow \{\mathit{true},\mathit{false}\}$ is a function 
that returns true if action and state variables satisfy 
global constraints, $T: \mathbb{R}^{|S|} \times \mathbb{R}^{|A|} \to \mathbb{R}^{|S|}$ denotes the 
stationary transition function, and $R: \mathbb{R}^{|S|} \times \mathbb{R}^{|A|} \rightarrow \mathbb{R}$ 
is the reward function. Finally, 
$I: \mathbb{R}^{|S|} \rightarrow \{\mathit{true},\mathit{false}\}$ represents the initial state constraints, 
and $G: \mathbb{R}^{|S|} \rightarrow \{\mathit{true},\mathit{false}\}$ represents the goal 
constraints. For horizon $H$, a solution $\pi = \langle \bar{A}^1, \dots, \bar{A}^H \rangle$ to problem 
$\Pi$ (i.e. a plan for $\Pi$) is a value assignment to the action variables with values 
$\bar{A}^{t} = \langle \bar{a}^{t}_1, \dots, \bar{a}^{t}_{|A|} \rangle \in \mathbb{R}^{|A|}$ for all time 
steps $t\in \{1,\dots,H\}$ (and state variables with values 
$\bar{S}^{t} = \langle \bar{s}^{t}_1, \dots, \bar{s}^{t}_{|S|} \rangle \in \mathbb{R}^{|S|}$ 
for all time steps $t\in \{1,\dots,H+1\}$) such that 
$T(\langle \bar{s}^{t}_1, \dots, \bar{s}^{t}_{|S|} , \bar{a}^{t}_1, \dots, \bar{a}^{t}_{|A|} \rangle) = \bar{S}^{t+1}$ 
and $C(\langle \bar{s}^{t}_1, \dots, \bar{s}^{t}_{|S|} , \bar{a}^{t}_1, \dots, \bar{a}^{t}_{|A|} \rangle) = true$ 
for all time steps $t\in \{1,\dots,H\}$, and the initial and goal state constraints are satisfied, i.e. 
$I(\bar{S}^{1}) = true$ and $G(\bar{S}^{H+1}) = true$, where $\bar{x}^{t}$ denotes the value of variable 
$x \in A \cup S$ at time step $t$. Similarly, an optimal solution to $\Pi$ is a plan such that the total reward 
$\sum_{t=1}^{H}R(\langle \bar{s}^{t+1}_1, \dots, \bar{s}^{t+1}_{|S|} , \bar{a}^{t}_1, \dots, \bar{a}^{t}_{|A|} \rangle)$ 
is maximized. For notational simplicity, we denote the tuple of variables $\langle x_{d_1}, \dots, x_{d_{|D|}} \rangle$ as $\langle x_d | d \in D \rangle$ given set $D$, and use the symbol ${}^\frown$ for 
the concatenation of two tuples. Given the notations and the description of the planning problem, 
we next describe a data-driven planning framework using learned NNs.

\subsection{Planning with Neural Network Learned Transition Models}

Hybrid Deep MILP Planning~\cite{Say2017} (HD-MILP-Plan) is a two-stage 
data-driven framework for learning and solving planning problems. Given samples 
of state transition data, the first stage of the HD-MILP-Plan process learns the transition 
function $\tilde{T}$ using a NN with Rectified Linear Units (ReLUs)~\cite{Nair2010} 
and linear activation units. In the second stage, the learned transition function 
$\tilde{T}$ is used to construct the learned planning problem 
$\tilde{\Pi} = \langle S,A,C,\tilde{T},I,G,R \rangle$. As shown in 
Figure~\ref{fig:hdmilpplan}, the learned transition function $\tilde{T}$ is 
sequentially chained over the horizon $t\in \{1,\dots,H\}$, and compiled into a MILP. 
Next, we review the MILP compilation of HD-MILP-Plan.

\begin{figure}
\centering
\includegraphics[width=\linewidth]{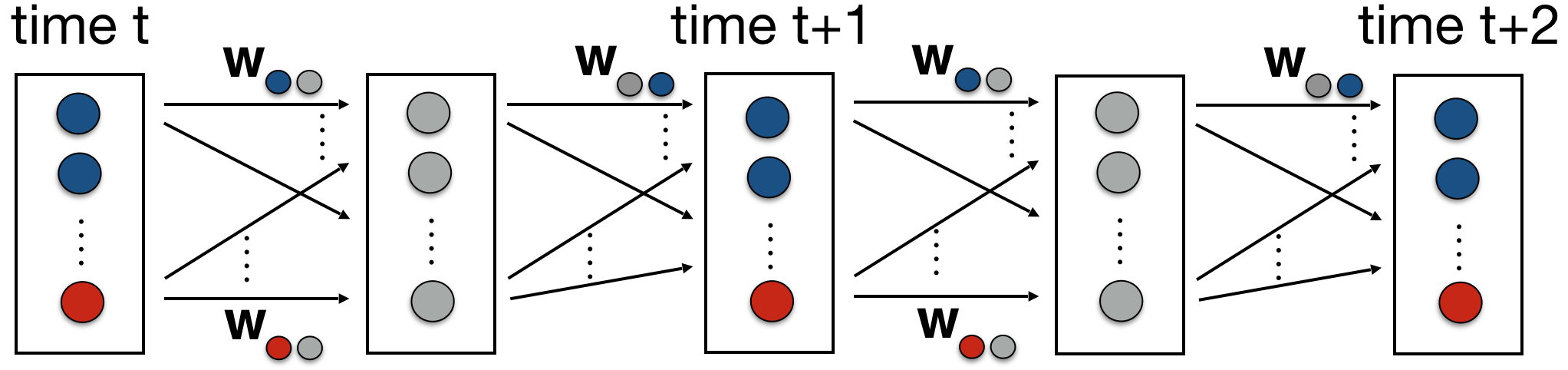}
\caption{Visualization of the learning and planning framework~\cite{Say2017}, 
where blue circles represent state variables $S$, red circles represent action 
variables $A$, gray circles represent ReLUs $U$ and $\textbf{w}$ 
represent the weights of a NN. During the learning stage, the weights $\textbf{w}$ 
are learned from data. In the planning stage, the weights are fixed and the planner 
optimizes a given total (cumulative) reward function with respect to the set of free 
action variables $A$ and state variables $S$.}
\label{fig:hdmilpplan}
\end{figure}

\subsection{Mixed-Integer Linear Programming Compilation of HD-MILP-Plan}

We begin with all notation necessary for HD-MILP-Plan.

\subsubsection{Parameters}

\begin{itemize}
\item $U$ is the set of ReLUs in the neural network.
\item $O$ is the set of output units in the neural network.
\item $w_{i,j}$ denotes the learned weight of the neural network 
between units $i$ and $j$.
\item $A(u)$ is the set of action variables connected as inputs to unit $u \in U \cup O$.
\item $S(u)$ is the set of state variables connected as inputs to unit $u \in U \cup O$.
\item $U(u)$ is the set of ReLUs connected as inputs to unit $u \in U \cup O$.
\item $O(s)$ specifies the output unit that predicts the value of state variable $s \in S$.
\item $B(u)$ is a constant representing the bias of unit $u \in U\cup O$.
\item $M$ is a large constant used in the big-M constraints.
\end{itemize}

\subsubsection{Decision Variables}

\begin{itemize}
\item ${X}_{a,t}$ is a decision variable with continuous domain denoting the value of action variable $a \in A$ at time step $t$.
\item ${Y}_{s,t}$ is a decision variable with continuous domain denoting the value of state variable $s\in S$ at time step $t$.
\item ${P}_{u,t}$ is a decision variable with continuous domain denoting the output of ReLU $u\in U$  at time step $t$.
\item ${P}^{b}_{u,t} = 1$ if ReLU $u\in U$ is activated 
at time step $t$, 0 otherwise (i.e., ${P}^{b}_{u,t}$ is a Boolean decision variable).
\end{itemize}

\subsubsection{MILP Compilation}

\begin{align}
&\text{maximize }\sum_{t=1}^{H} R(\langle Y_{s,t+1} | s \in S \rangle {}^\frown \langle X_{a,t} | a\in A \rangle)\label{ref:HD_0}\\
&\text{subject to} \nonumber\\
&I(\langle Y_{s,1} |s \in S \rangle) \label{ref:HD_1}\\
&C(\langle Y_{s,t} | s \in S \rangle {}^\frown \langle X_{a,t} | a \in A \rangle) \label{ref:HD_2}\\
&G(\langle Y_{s,H+1}|s \in S \rangle) \label{ref:HD_3}\\
&{P}_{u,t} \leq M {P}^{b}_{u,t} \quad \forall{u \in U}\label{ref:HD_4}\\
&{P}_{u,t} \leq M (1-{P}^{b}_{u,t}) + \mathit{In}(u,t) \quad \forall{u \in U}\label{ref:HD_5}\\
&{P}_{u,t}\geq \mathit{In}(u,t) \quad \forall{u \in U}\label{ref:HD_6}\\
&{Y}_{s,t+1} = \mathit{In}(u,t) \quad \forall{u \in O(s), s \in S}\label{ref:HD_7}
\end{align}
for all time steps $t=1,\dots,H$ except for constraints (\ref{ref:HD_1})-(\ref{ref:HD_3}). Expression $\mathit{In}(u,t)$ denotes the total weighted input of unit $u \in U\cup O$ at time step $t$, and is equivalent to $B(u) + \sum_{u'\in U(u)}{w_{u',u}}{P}_{u',t} + \sum_{s \in S(u)}{w_{s,u}}{Y}_{s,t} + \sum_{a \in A(u)}{w_{a,u}}{X}_{a,t}$.

In the above MILP, the objective function (\ref{ref:HD_0}) maximizes the sum of 
rewards over a given horizon $H$. Constraints (\ref{ref:HD_1}-\ref{ref:HD_3}) 
ensure the initial state, global and goal state constraints are satisfied. Constraints 
(\ref{ref:HD_4}-\ref{ref:HD_7}) model the learned transition function $\tilde{T}$. Note that while 
constraints (\ref{ref:HD_4}-\ref{ref:HD_6}) are sufficient to encode the piecewise linear 
activation behaviour of ReLUs, the use of big-M constraints (\ref{ref:HD_4}-\ref{ref:HD_5}) can hinder the overall performance 
of the underlying B\&B solvers that rely on the linear relaxation of the MILP. 
Therefore next, we turn to potential heuristics that will be used to strengthen the MILP 
compilation of HD-MILP-Plan.

\subsection{Potential Heuristics}

Potential heuristics~\cite{Pommerening2015,Jendrik2015} are a family of heuristics that map 
a set of features to their numerical potentials. In the context of cost-optimal 
classical planning, the heuristic value of a state is defined as the sum of 
potentials for all the features that are true in that state. Potential heuristics 
provide an efficient method for computing a lower bound on the cost of a given state.

In this paper, we introduce an alternative use of potential functions to tighten the 
linear relaxation of ReLU units in our HD-MILP-Plan compilation and improve the search 
efficiency of the underlying B\&B solver.
We define the features of the learned NN over its set of hidden units 
$U$ (i.e., gray circles in Figure~\ref{fig:hdmilpplan}), and compute the potential 
contribution (i.e., reward potential) of each hidden unit $u \in U$ to the reward 
function $R$ for any time step $t$. These reward potentials are then used to introduce 
additional constraints on ReLU activations that help guide B\&B search in HD-MILP-Plan. 
Specifically, we are interested in finding a set of reward potentials, denoted as $v^{on}_u$ 
and $v^{off}_u$ representing the activation (i.e., ${P}^{b}_{u,t} = 1$) and the deactivation 
(i.e., ${P}^{b}_{u,t} = 0$) of ReLUs $u\in U$, such that the relation 
$\sum_{u\in U}v^{on}_u {P}^{b}_{u,t} + v^{off}_u (1-{P}^{b}_{u,t}) \geq 
R(\langle Y_{s,t+1} | s \in S \rangle {}^\frown \langle X_{a,t}| a \in A \rangle)$ holds for all feasible values of 
${P}^{b}_{u,t}$, $Y_{s,t+1}$ and $X_{a,t}$ at any time step $t$. Once values $\bar{v}^{on}_u$ 
and $\bar{v}^{off}_u$ are computed, we will add 
$\sum_{u\in U}\bar{v}^{on}_u {P}^{b}_{u,t} + \bar{v}^{off}_u (1-{P}^{b}_{u,t}) \geq 
R(\langle Y_{s,t+1} | s \in S \rangle {}^\frown \langle X_{a,t}| a\in A \rangle)$ as a linear constraint to strengthen HD-MILP-Plan. 
Next we describe bilevel programming that we use to model the problem of finding optimal 
reward potentials.

\subsection{Bilevel Programming}

Bilevel programming~\cite{Bard2000} is an optimization framework for modeling 
two-level asymetrical decision making problems with a leader and a follower problem
where the leader has complete knowledge of the follower, and the follower 
only observes the decisions of the leader to make an optimal decision. Therefore, 
the leader must incorporate the optimal decision of the follower to optimize its 
objective. 

In this work, we use bilevel programming to compactly model the problem of 
finding the optimal reward potentials that has exponential number of constraints. 
In the bilevel programming description of 
the optimal reward potentials problem, the leader selects the optimal values $\bar{v}^{on}_u$ and $\bar{v}^{off}_u$ of 
reward potentials, and the follower selects the values 
of ${P}^{b}_{u,t}$, $Y_{s,t+1}$ and $X_{a,t}$ such that the expression 
$R(\langle Y_{s,t+1} | s \in S \rangle {}^\frown \langle X_{a,t}| a \in A \rangle) - \sum_{u\in U}v^{on}_u {P}^{b}_{u,t} + 
v^{off}_u (1-{P}^{b}_{u,t})$ is maximized. That is, the follower tries to find 
values of ${P}^{b}_{u,t}$, $Y_{s,t+1}$ and $X_{a,t}$ that violate the relation 
$\sum_{u\in U}v^{on}_u {P}^{b}_{u,t} + v^{off}_u (1-{P}^{b}_{u,t}) \geq 
R(\langle Y_{s,t+1} | s \in S \rangle {}^\frown \langle X_{a,t}| a\in A \rangle)$ as much as possible. Therefore the 
leader must select the values $\bar{v}^{on}_u$ and $\bar{v}^{off}_u$ of reward potentials by 
incorporating the optimal decision making model of the follower. Next, we describe the reward potentials for learned NNs.

\section{Reward Potentials for Learned Neural Networks}

In this section, we present the optimal reward potentials problem 
and an efficient constraint generation framework for finding 
reward potentials for learned NNs.

\subsection{Optimal Reward Potentials Problem}

The problem of finding the optimal reward potentials over a set of ReLUs $U$ 
for any time step $t$ can be defined as the following bilevel 
optimization problem.
\subsubsection{Leader Problem}
\begin{align}
&\min_{v^{on}_u, v^{off}_{u}, Y_{s,t}, Y_{s,t+1}, X_{a,t}, {P}^{b}_{u,t}}\sum_{u\in U}v^{on}_u + v^{off}_u\label{ref:L_0}\\
&\text{subject to }\nonumber\\
&\sum_{u\in U}v^{on}_u {P}^{b}_{u,t} + v^{off}_u (1-{P}^{b}_{u,t}) \geq R(\langle Y_{s,t+1} | s \in S \rangle {}^\frown \langle X_{a,t}| a \in A \rangle)\label{ref:L_1}\\
&Y_{s,t}, Y_{s,t+1}, X_{a,t}, {P}^{b}_{u,t} \in \arg \text{Follower Problem}\nonumber
\end{align}
\subsubsection{Follower Problem}
\begin{align}
&\max_{Y_{s,t}, Y_{s,t+1}, X_{a,t}, {P}^{b}_{u,t}} R(\langle Y_{s,t+1} | s \in S \rangle {}^\frown \langle X_{a,t}| a \in A \rangle) - \sum_{u\in U}v^{on}_u {P}^{b}_{u,t} + v^{off}_u (1-{P}^{b}_{u,t})\label{ref:F_0}\\
&\text{subject to }\nonumber\\
&\text{Constraints (\ref{ref:HD_2})\text{ and }(\ref{ref:HD_4}-\ref{ref:HD_7})}\nonumber
\end{align}

In the above bilevel problem, the leader problem selects the values $\bar{v}^{on}_u$ and $\bar{v}^{off}_u$ of the reward 
potentials such that their total sum is minimized 
(i.e., objective function (\ref{ref:L_0})\footnote{The objective function 
(\ref{ref:L_0}) is similar to the objective function of  ``All Syntactic States'' for potential heuristics used in classical planning~\cite{Jendrik2015}.}), and their total weighted sum for all ReLU 
activations is an upper bound to all values of the reward function $R$ (i.e., constraint 
(\ref{ref:L_1}) and the follower problem). Given the values $\bar{v}^{on}_u$ and $\bar{v}^{off}_u$ of the reward potentials, the follower selects the values of decision 
variables $Y_{s,t}$, $Y_{s,t+1}$, $X_{a,t}$ and ${P}^{b}_{u,t}$ such that the difference between 
the value of the reward function $R$ and the sum of reward potentials is maximized subject 
to constraints (\ref{ref:HD_2}) and (\ref{ref:HD_4}-\ref{ref:HD_7}). Next, we show the 
correctness of the optimal reward potentials problem as the bilevel program described by the leader (i.e., objective function (\ref{ref:L_0}) and constraint (\ref{ref:L_1})) and the follower (i.e., objective function (\ref{ref:F_0}) and constraints (\ref{ref:HD_2}) and (\ref{ref:HD_4}-\ref{ref:HD_7})) problems.

\begin{theorem}[Correctness of The Optimal Reward Potentials Problem]
\label{proof:gac}
Given constraints (\ref{ref:HD_2}) and 
(\ref{ref:HD_4}-\ref{ref:HD_7}) are feasible, the optimal reward potentials problem finds the values $\bar{v}^{on}_u$ and $\bar{v}^{off}_u$ of reward potentials 
such that the relation 
$\sum_{u\in U}\bar{v}^{on}_u {P}^{b}_{u,t} + \bar{v}^{off}_u (1-{P}^{b}_{u,t}) \geq 
R(\langle Y_{s,t+1} | s \in S \rangle{}^\frown \langle X_{a,t}| a \in A \rangle)$ holds for all values of ${P}^{b}_{u,t}$, $Y_{s,t+1}$ 
and $X_{a,t}$ at any time step $t$.
\end{theorem}

\begin{proof}[by Contradiction]
Let $\bar{v}^{on}_u$ and $\bar{v}^{off}_u$ denote the values of reward potentials 
selected by the leader problem that violate the relation 
$\sum_{u\in U}\bar{v}^{on}_u {P}^{b}_{u,t} + \bar{v}^{off}_u (1-{P}^{b}_{u,t}) \geq 
R(\langle Y_{s,t+1} | s \in S \rangle {}^\frown \langle X_{a,t}| a \in A \rangle)$ for some values 
$\bar{Y}_{s,t+1}$, $\bar{X}_{a,t}$ and $\bar{P}^{b}_{u,t}$, implying 
$R(\langle \bar{Y}_{s,t+1} | s \in S \rangle {}^\frown \langle \bar{X}_{a,t}| a \in A \rangle) - 
\sum_{u\in U}\bar{v}^{on}_u \bar{P}^{b}_{u,t} + \bar{v}^{off}_u (1-\bar{P}^{b}_{u,t}) > 0$. 
However, the feasibility of constraint (\ref{ref:L_1}) implies that the value of the 
objective function ($\ref{ref:F_0}$) must be non-positive (i.e., the follower 
problem is not solved to optimality), which yields the desired contradiction.
\end{proof}

Note that we omit the case when constraints (\ref{ref:HD_2}) and 
(\ref{ref:HD_4}-\ref{ref:HD_7}) are infeasible because it implies the 
infeasibility of the learned planning problem $\tilde{\Pi}$. 
Next, we describe a finite-time constraint generation algorithm for computing 
reward potentials.

\subsection{Constraint Generation for Computing Reward Potentials}

The optimal reward potentials problem can be solved efficiently through 
the following constraint generation framework that decomposes the problem 
into a master problem and a subproblem.\footnote{As noted by our reviewers, 
our constraint generation framework is related to Counterexample-guided 
Abstraction Refinement (CEGAR)~\cite{Clarke2000}. The clear differences between 
the typical use of CEGAR and our work are: (i) problem formalizations (i.e., bilevel 
programming versus iterative model-checking) and (ii) purposes (i.e., obtaining 
valid bounds on planning reward function $R$ versus verification of an abstract 
model). Naturally, what constitutes a violation is also different (i.e., error on 
reward estimation versus a spurious counterexample).} The master problem finds 
the values $\bar{v}^{on}_u$ and $\bar{v}^{off}_u$ of ReLU potential variables. The 
subproblem finds the values $\bar{P}^{b}_{u,t}$ of ReLU variables that violate 
constraint (\ref{ref:L_1}) the most for given values $\bar{v}^{on}_u$ and $\bar{v}^{off}_u$, 
and also finds the maximum value of reward function $R$ for given $\bar{P}^{b}_{u,t}$ 
which is denoted as $R^{*}(\langle \bar{P}^{b}_{u,t} | u\in U \rangle)$. 
Intuitively, the master problem selects the values $\bar{v}^{on}_u$ and $\bar{v}^{off}_u$ of ReLU 
potentials that are checked by the 
subproblem for the validity of the relation 
$\sum_{u\in U}\bar{v}^{on}_u {P}^{b}_{u,t} + \bar{v}^{off}_u (1-{P}^{b}_{u,t}) \geq 
R(\langle Y_{s,t+1} | s \in S \rangle{}^\frown \langle X_{a,t}| a \in A \rangle)$ for all feasible values of 
${P}^{b}_{u,t}$, $Y_{s,t+1}$ and $X_{a,t}$ at any time step $t$. If a violation is found, 
a linear constraint corresponding to a given $\bar{P}^{b}_{u,t}$ and 
$R^{*}(\langle \bar{P}^{b}_{u,t} | u\in U \rangle)$ is added back to the master problem and the 
procedure is repeated until no violation is found by the subproblem.

\subsubsection{Subproblem $\mathcal{S}$:}

For a complete value assignment $\bar{v}^{on}_u$ 
and $\bar{v}^{off}_u$ to ReLU potential variables, the subproblem optimizes the violation (i.e., objective 
function (\ref{ref:F_0})) with respect to constraints 
(\ref{ref:HD_2}) and (\ref{ref:HD_4}-\ref{ref:HD_7}) as follows.
\begin{align}
&\max_{Y_{s,t}, Y_{s,t+1}, X_{a,t}, {P}^{b}_{u,t}} R(\langle Y_{s,t+1} | s \in S \rangle {}^\frown \langle X_{a,t}| a \in A \rangle) - \sum_{u\in U}\bar{v}^{on}_u {P}^{b}_{u,t} + \bar{v}^{off}_u (1-{P}^{b}_{u,t})\label{ref:Sub_0}\\
&\text{subject to }\nonumber\\
&\text{Constraints (\ref{ref:HD_2})\text{ and }(\ref{ref:HD_4}-\ref{ref:HD_7})}\nonumber
\end{align}

We denote the optimal values of ReLU variables ${P}^{b}_{u,t}$, found by solving the 
subproblem as $\bar{P}^{b}_{u,t}$, and the value of the reward function $R$ 
found by solving the subproblem as $R^{*}(\langle \bar{P}^{b}_{u,t} | u\in U \rangle)$. Further, 
we refer to subproblem as $\mathcal{S}$.

\subsubsection{Master problem $\mathcal{M}$:} 

Given the set of complete value assignments $K$ to ReLU variables with values $\bar{P}^{b,k}_{u,t}$ 
and optimal objective values $R^{*}(\langle\bar{P}^{b,k}_{u,t} | u\in U\rangle)$ for all $k\in K$, 
the master problem optimizes the regularized\footnote{The squared terms penalize arbitrarily 
large values of potentials to avoid numerical issues. A similar numerical issue has been found 
in the computation of potential heuristics for cost-optimal classical planning problems with 
dead-ends~\cite{Jendrik2015}.} sum of reward potentials (i.e., regularized objective function 
(\ref{ref:L_0})) with respect to the modified version of constraint (\ref{ref:L_1}) as follows.
\begin{align}
&\min_{v^{on}_u, v^{off}_u}\sum_{u\in U}v^{on}_u + v^{off}_u + \lambda \sum_{u\in U}(v^{on}_u)^{2} + (v^{off}_u)^{2}\label{ref:M_0}\\
&\text{subject to }\nonumber\\
&\sum_{u\in U}v^{on}_u \bar{P}^{b,k}_{u,t} + v^{off}_u (1-\bar{P}^{b,k}_{u,t}) \geq R^{*}(\langle\bar{P}^{b,k}_{u,t} | u\in U\rangle) \quad \forall{k\in K}\label{ref:M_1}
\end{align}

We denote the optimal values of ReLU potential variables $v^{on}_{u}$ and  $v^{off}_{u}$, 
found by solving the master problem as $\bar{v}^{on}_{u}$ and $\bar{v}^{off}_{u}$, 
respectively. Further, we refer to 
master problem as $\mathcal{M}$.

\subsubsection{Reward Potentials Algorithm} 

Given the definitions of the master problem $\mathcal{M}$ and the subproblem 
$\mathcal{S}$, the constraint generation algorithm for computing 
an optimal reward potential is outlined as follows.

\begin{algorithm}
\footnotesize
\caption{Reward Potentials Algorithm}\label{rewpotalg}
\begin{algorithmic}[1]
\State $\textit{k} \gets 1$, violation $\gets \infty$, $\mathcal{M} \gets$ objective function (\ref{ref:M_0})
\While {violation $>$ 0}
\State $\bar{v}^{on}_u, \bar{v}^{off}_u \gets \mathcal{M}$
\State $\bar{P}^{b,k}_{u,t}, \bar{Y}_{s,t+1}, \bar{X}_{a,t}, R^{*}(\langle\bar{P}^{b,k}_{u,t} | u\in U\rangle) \gets \mathcal{S}(\bar{v}^{on}_u, \bar{v}^{off}_u)$
\State violation = $R(\langle \bar{Y}_{s,t+1} | s \in S \rangle {}^\frown \langle \bar{X}_{a,t}| a \in A \rangle) - \sum_{u\in U}\bar{v}^{on}_u \bar{P}^{b,k}_{u,t} + \bar{v}^{off}_u (1-\bar{P}^{b,k}_{u,t})$
\State $\mathcal{M} \gets \mathcal{M} \cup \sum_{u\in U}v^{on}_u \bar{P}^{b,k}_{u,t} + v^{off}_u (1-\bar{P}^{b,k}_{u,t}) \geq R^{*}(\langle\bar{P}^{b,k}_{u,t} | u\in U\rangle)$ (i.e., update constraint (\ref{ref:M_1}))
\State $k \gets k + 1$
\EndWhile
\end{algorithmic}
\end{algorithm}

Given constraints (\ref{ref:HD_2}) and (\ref{ref:HD_4}-\ref{ref:HD_7}) are feasible, 
Algorithm~\ref{rewpotalg} iteratively computes reward potentials 
$v^{on}_{u}$ and  $v^{off}_{u}$ (i.e., line 3), and first checks if there exists an 
activation pattern, that is a complete value assignment $\bar{P}^{b,k}_{u,t}$ to ReLU variables, 
that violates constraint (\ref{ref:L_1}) (i.e., lines 4 and 5), and then returns the optimal 
reward value $R^{*}(\langle\bar{P}^{b,k}_{u,t} | u\in U\rangle)$ for the violating activation pattern. 
Given the optimal reward value $R^{*}(\langle\bar{P}^{b,k}_{u,t} | u\in U\rangle)$ for the violating 
activation pattern, constraint (\ref{ref:M_1}) is updated (i.e., lines 6-7). Since there are 
finite number of activation patterns and solving $\mathcal{S}$ gives the maximum value of 
$R^{*}(\langle \bar{P}^{b,k}_{u,t} | u\in U \rangle)$ for each pattern $k \in \{1, \dots, 2^{|U|}\}$, the Reward 
Potentials Algorithm~\ref{rewpotalg} terminates in at most $k \leq 2^{|U|}$ iterations with 
an optimal reward potential for the learned NN.

\subsubsection{Increasing the Granularity of the Reward Potentials Algorithm} 

The feature space of Algorithm~\ref{rewpotalg} can be enhanced
to include information on each ReLUs input and/or output. Instead of computing reward 
potentials for only the activation $\bar{v}^{on}_{u}$ and deactivation 
$\bar{v}^{off}_{u}$ of ReLU $u\in U$, we (i) introduce an interval parameter $N$ to 
split the output range of each ReLU $u$ into $N$ equal size intervals, (ii) introduce 
auxiliary Boolean decision variables ${P'}^{b}_{i,u,t}$ to represent the activation interval 
of ReLU $u$ such that ${P'}^{b}_{i,u,t}=1$ if and only if the output of ReLU $u$ is within 
interval $i\in \{1,\dots,N\}$, and ${P'}^{b}_{i,u,t}=0$ otherwise, and (iii) compute reward 
potentials for each activation interval $\bar{v}^{on}_{u,1}, \dots, \bar{v}^{on}_{u,N}$ 
and deactivation $\bar{v}^{off}_{u}$ of ReLU $u\in U$.

\subsection{Strengthening HD-MILP-Plan}

Given optimal reward potentials $\bar{v}^{on}_{u,1}, \dots, \bar{v}^{on}_{u,N}$ and 
$\bar{v}^{off}_{u}$, the MILP compilation of HD-MILP-Plan is strengthened through the 
addition of following constraints:

\begin{align}
&\sum_{u\in U}\sum^{N}_{i=1}\bar{v}^{on}_{u,i} {P'}^{b}_{i,u,t} + \bar{v}^{off}_u (1-x^{t}_{u}) \geq R(\langle Y_{s,t+1} | s \in S \rangle{}^\frown \langle X_{a,t}| a \in A \rangle) \label{ref:S_0}\\
&\sum^{N}_{i=1}{P'}^{b}_{i,u,t} = {P}^{b}_{u,t} \label{ref:S_1}\\
&N_u\frac{(i-1)}{N} {P'}^{b}_{i,u,t} \leq {P}_{u,t} \leq N_u - (N_u - N_u\frac{i}{N}) {P'}^{b}_{i,u,t} \quad \forall{i\in \{1,\dots,N\}, u\in U} \label{ref:S_2}
\end{align}
for all time steps $t\in \{1,\dots,H\}$ where $N_u$ denotes the upperbound obtained 
from performing forward reachability on the output of each ReLU $u\in U$ 
in the learned NN. Briefly, constraint (\ref{ref:S_0}) provides the upperbound on 
the reward function $R$ as a function of ReLU activation intervals and deactivations. 
Constraint (\ref{ref:S_1}) ensures that (i) at most one 
auxillary variable ${P'}^{b}_{i,u,t}$ is selected, and (ii) at least one 
auxillary variable ${P'}^{b}_{i,u,t}$ is selected if and only if ReLU $u$ is 
activated. Constraint (\ref{ref:S_2}) ensures that the output of each ReLU is within 
its selected activation interval. Next, we present our experimental results 
to demonstrate the efficiency and the utility of computing reward potential and 
strengthening HD-MILP-Plan.

\section{Experimental Results}
In this section, we present computational results on (i) the convergence of 
Algorithm~\ref{rewpotalg}, and (ii) the overall strengthening of HD-MILP-Plan 
with the addition of constraints (\ref{ref:S_0}-\ref{ref:S_2}) for the task 
of planning over long horizons. First, we present results on the overall 
efficiency of Algorithm~\ref{rewpotalg} and the strengthening of HD-MILP-Plan 
over multiple learned planning instances. Then, we focus on the most 
computationally expensive domain identified by our experiments to further investigate 
the convergence behaviour of Algorithm~\ref{rewpotalg} and the overall strengthening 
of HD-MILP-Plan as a function of time.

\subsection{Experimental Setup}

The experiments were run on a MacBookPro with 2.8 GHz Intel Core i7 16GB memory. 
All instances and the respective learned neural networks from the HD-MILP-Plan 
paper~\cite{Say2017}, namely \textit{Navigation}, \textit{Reservoir Control} and 
\textit{HVAC}~\cite{Say2017}, were 
selected.\footnote{https://github.com/saybuser/HD-MILP-Plan} Both domain instance 
sizes and their 
respective learned NN sizes are detailed in Table~\ref{tab:domains} where columns 
from left to right denote the name of problem instances, the structures of the learned 
NNs where each number denotes the width of a layer and the values of the planning horizon 
$H$, respectively. The range 
bounds on action variables for Navigation domains were constrained to $[-0.1, 0.1]$. 
CPLEX 12.9.0~\cite{IBM2019} solver was used to optimize both 
Algorithm~\ref{rewpotalg}, and HD-MILP-PLan, with 6000 seconds 
of total time limit per domain instance. In our experiments, we show results for 
the base model (i.e., objective (\ref{ref:HD_0}) and constraints 
(\ref{ref:HD_1}-\ref{ref:HD_7})) and the strengthened model with the addition of 
constraints (\ref{ref:S_0}-\ref{ref:S_2}) for the values of interval parameter 
$N=2,3$.\footnote{The preliminary experimental results for interval parameter $N=1$ 
have not shown significant improvements over the base encoding of HD-MILP-Plan.} 
Finally in the master problem, we have chosen 
the regularizer constant $\lambda$ in the objective function (9) to be 
$\frac{1}{\sqrt{M}}$ where $M$ is the large constant used in the big-M constraints 
of HD-MILP-Plan (i.e., constraints (\ref{ref:HD_4}-\ref{ref:HD_5})).

\begin{table}
  \centering
  \caption{Domain and learned NN descriptions where columns 
from left to right denote the name of problem instances, the structures of NNs used to learn each transition model $\tilde{T}$ where each number denotes the width of a layer, and the values of the planning horizon 
$H$, respectively.}
  \label{tab:domains}
  \begin{tabular}{| l | c | c |}
    \hline
    Domain Instance & Network Structure & Horizon \\
    \hline
    Navigation (8-by-8 maze) & 4:32:32:2 & 100 \\ \hline
    Navigation (10-by-10 maze) & 4:32:32:2 & 100 \\ \hline
    Reservoir Control (3 reservoirs) & 6:32:3 & 500 \\ \hline
    Reservoir Control (4 reservoirs) & 8:32:4 & 500 \\ \hline
    HVAC (3 rooms) & 6:32:3 & 100 \\ \hline
    HVAC (6 rooms) & 12:32:6 & 100 \\ \hline
  \end{tabular}
\end{table}

\subsection{Overall Results}

In this section, we present the experimental results on (i) the computation of 
the optimal reward potentials using Algorithm~\ref{rewpotalg}, (ii) and the performance 
of HD-MILP-Plan with the addition of constraints (\ref{ref:S_0}-\ref{ref:S_2}) 
over multiple learned planning instances over long horizons. Table~\ref{tab:overall} 
summarizes the computational results and highlights the best performing 
HD-MILP-Plan settings for each learned planning instance.

\begin{table}
  \centering
  \caption{Summary of experimental results on the computationally efficiency 
  of Algorithm~\ref{rewpotalg} and HD-MILP-Plan with the addition of constraint 
  (\ref{ref:S_0}-\ref{ref:S_2}) over multiple learned planning instances with long horizons.}
  \label{tab:overall}
  \begin{tabular}{| l | c | c | c | c | c | c | c |}
    \hline
    Domain Setting & Alg.~\ref{rewpotalg} & Cumul. & Primal & Dual & Open & Closed \\
    \hline
    Nav,8,100,Base & - & 6000 & - & -261.4408 & 16536 & 27622 \\ \hline
    \textbf{Nav,8,100,N=2} & 345 & 6000 & - & \textbf{-267.1878} & 6268 & 15214 \\ \hline
    Nav,8,100,N=3 & 1150 & 6000 & - & -267.056 & 6189 & 12225 \\ \hline
    Nav,10,100,Base & - & 6000 & - & -340.5974 & 17968 & 35176 \\ \hline
    \textbf{Nav,10,100,N=2} & 800 & 6000 & - & \textbf{-340.6856} & 14435 & 27651 \\ \hline
    Nav,10,100,N=3 & 1700 & 6000 & - & -339.8124 & 2593 & 7406 \\ \hline
    HVAC,3,100,Base & - & 260.21 & Opt. found & Opt. proved & 0 & 289529 \\ \hline
    \textbf{HVAC,3,100,N=2} & 7 & \textbf{88.21} & Opt. found & Opt. proved & 0 & 2501 \\ \hline
    HVAC,3,100,N=3 & 9 & 194.44 & Opt. found & Opt. proved & 0 & 10891 \\ \hline
    HVAC,6,100,Base & - & 6000 & -1214369.086 & -1213152.304 & 618687 & 648207 \\ \hline
    \textbf{HVAC,6,100,N=2} & 8 & 6000 & -1214365.427 & \textbf{-1213199.787} & 554158 & 567412 \\ \hline
    \textbf{HVAC,6,100,N=3} & 10 & 6000 & \textbf{-1214364.704} & -1213025.189 & 1011348 & 1021637 \\ \hline
    \textbf{Res,3,500,Base} & - & \textbf{33.01} & Opt. found & Opt. proved & 0 & 1 \\ \hline
    Res,3,500,N=2 & 1 & 99.81 & Opt. found & Opt. proved & 0 & 714 \\ \hline
    Res,3,500,N=3 & 2 & 90.27 & Opt. found & Opt. proved & 0 & 674 \\ \hline
    Res,4,500,Base & - & 300.71 & Opt. found & Opt. proved & 0 & 1236 \\ \hline
    \textbf{Res,4,500,N=2} & 7 & \textbf{109.66} & Opt. found & Opt. proved & 0 & 1924 \\ \hline
    Res,4,500,N=3 & 6 & 232.19 & Opt. found & Opt. proved & 0 & 1294 \\ \hline
  \end{tabular}
\end{table}

The first column of Table~\ref{tab:overall} identifies the domain setting of each row. 
The second column denotes the runtime of Algorithm~\ref{rewpotalg} in seconds. The third column (i.e., Cumul.) denotes the cumulative runtime of Algorithm~\ref{rewpotalg} and HD-MILP-Plan 
in seconds. The remaining 
columns provide information on the performance of HD-MILP-Plan. Specifically, the fourth column 
(i.e., Primal) denotes the value of the incumbent plan found by HD-MILP-Plan, the fifth 
column (i.e., Dual) denotes the value of the duality bound found by HD-MILP-Plan, and the 
sixth and seventh columns (i.e., Open and Closed) denote the number of open and closed nodes in 
the B\&B tree respectively. The bolded values indicate the best performing HD-MILP-Plan settings for each 
learned planning instance where the performance of each setting is evaluated first based on the 
runtime performance (i.e., Cumul. column), followed by the quality of incumbent plan 
(i.e., Primal column) and duality bound (i.e., Dual column) obtained by HD-MILP-Plan.

In total of five out of six instances, we observe that strengthened HD-MILP-Plan with interval 
parameter $N=2$ performed the best. The pairwise comparison of the base HD-MILP-Plan and strengthened 
HD-MILP-Plan with interval parameter $N=3$ shows that in almost all instances, the strengthened 
model performed better in comparison to the base model. The only instance in which the base model 
significantly outperformed the other two was the Reservoir Control domain with three reservoirs where 
the B\&B solver was able to find an optimal plan in the root node. Overall, we found that especially 
in the instances where the optimality was hard to prove within the runtime limit of 6000 seconds 
(i.e., all Navigation instances and HVAC domain with 6 rooms), strengthened HD-MILP-Plan explored 
signigicantly less number of nodes in general while obtaining either higher quality incumbent plans 
or lower dual bounds. We observe that Algorithm~\ref{rewpotalg} terminated with optimal reward 
potentials in less than 10 seconds in both Reservoir Control and HVAC domains, and took as 
much as 1700 seconds in Navigation domain -- highlighting the effect of NN size and 
complexity (i.e., detailed in Table~\ref{tab:domains}) on the runtime of Algorithm~\ref{rewpotalg}.
As a result, next we focus on the most computationally expensive domain identified by our experiments, 
namely Navigation, to get a better understanding on the convergence behaviour of 
Algorithm~\ref{rewpotalg} and the overall efficiency of HD-MILP-Plan as a function of time.

\subsection{Detailed Convergence Results on Navigation Domain}

In this section, we inspect the convergence of Algorithm~\ref{rewpotalg} 
in the Navigation domain for computing an optimal reward potential for the learned NNs.

\begin{figure}[t!]
    \centering
    \includegraphics[width=.8\linewidth]{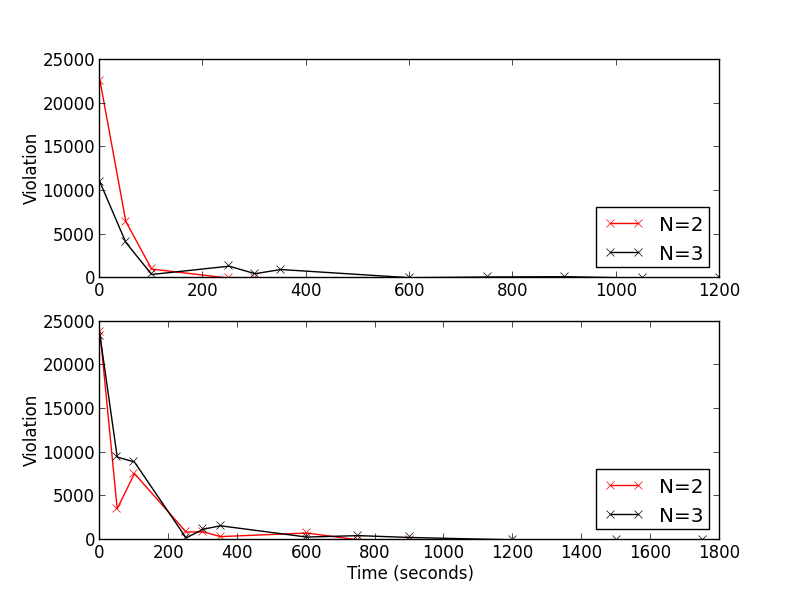}
  \caption{Convergence of Algorithm~\ref{rewpotalg} as 
  a function of time for the learned NNs of both Navigation 8-by-8 (i.e., top) and 
  Navigation 10-by-10 (i.e., bottom) planning instances. The violation of constraint (\ref{ref:L_1}) decreases exponentially as a function of time, showcasing a long-tail runtime behaviour and terminates 
with optimal reward potentials.}
  \label{fig:Viol_Nav}
\end{figure}

Figure~\ref{fig:Viol_Nav} visualizes the violation of constraint (\ref{ref:L_1}) as a function 
of time over the computation of optimal reward potentials using the Reward Potentials 
Algorithm~\ref{rewpotalg} for the learned NNs of both Navigation 8-by-8 (i.e., top) 
and Navigation 10-by-10 (i.e., bottom) planning instances. In both, we 
observe that the violation of constraint (\ref{ref:L_1}) decreases exponentially as 
a function of time, showcasing a long-tail runtime behaviour and terminates 
with optimal reward potentials.

\subsection{Detailed Strengthening Results on Navigation Domain}

Next, we inspect the overall strengthening of 
HD-MILP-Plan with respect to its underlying linear relaxation and search 
efficiency as a result of constraints (\ref{ref:S_0}-\ref{ref:S_2}), for the 
task of planning over long horizons in the Navigation domain.

\begin{figure}[t!]
\centering
\includegraphics[width=.8\linewidth]{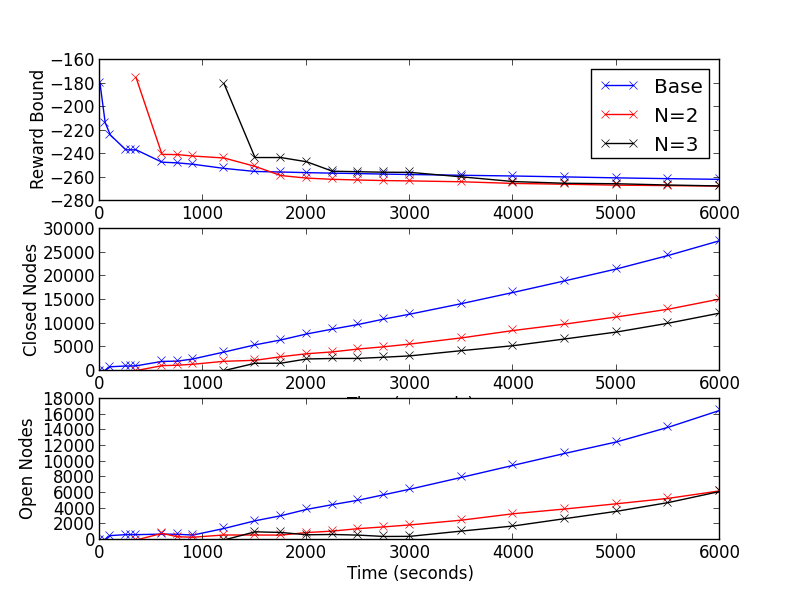}
\caption{Linear relaxation and search efficiency comparisons in Navigation domain 
with an 8-by-8 maze between the base and the strengthened HD-MILP-Plan using 
Algorithm~\ref{rewpotalg} with interval parameter $N=2,3$. Overall, we observe 
that HD-MILP-Plan with constraints (\ref{ref:S_0}-\ref{ref:S_2}) outperforms 
the base HD-MILP-Plan by 1700 and 3300 seconds with interval parameter $N=2,3$, 
respectively.}
\label{fig:Strength_Nav_8}
\end{figure}
\begin{figure}[t!]
\centering
\includegraphics[width=.8\linewidth]{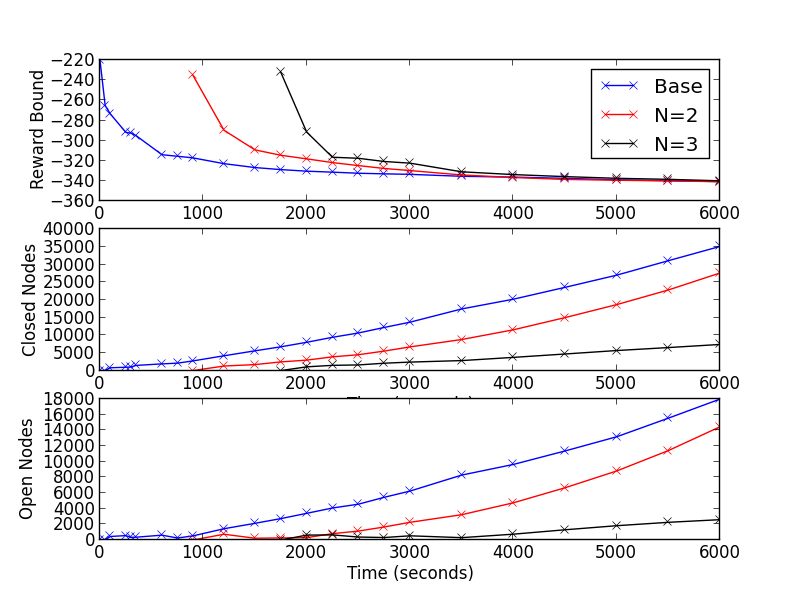}
\caption{Linear relaxation and search efficiency comparisons in Navigation domain 
with an 10-by-10 maze between the base and the strengthened HD-MILP-Plan using 
Algorithm~\ref{rewpotalg} with interval parameter $N=2,3$. Overall, we observe that 
HD-MILP-Plan with constraints (\ref{ref:S_0}-\ref{ref:S_2}) obtains a tighter bound 
compared to the base HD-MILP-Plan by 3750 seconds and reaches almost the same bound 
by the time limit (i.e., 6000 seconds) with interval parameter $N=2,3$, respectively.}
\label{fig:Strength_Nav_10}
\end{figure}

Figures~\ref{fig:Strength_Nav_8} and~\ref{fig:Strength_Nav_10} visualize the overall effect of incorporating 
constraints (\ref{ref:S_0}-\ref{ref:S_2}) into HD-MILP-Plan as a function of time 
for the Navigation domain 
with (a) 8-by-8 and (b) 10-by-10 maze sizes. In both Figures~\ref{fig:Strength_Nav_8} and~\ref{fig:Strength_Nav_10}, linear relaxation (i.e. top), number of closed nodes (i.e., middle), 
and number open nodes (i.e., bottom), are displayed as a function of time. The 
inspection of both Figures~\ref{fig:Strength_Nav_8} and~\ref{fig:Strength_Nav_10} show that once the reward 
potentials are computed, the addition of constraints (\ref{ref:S_0}-\ref{ref:S_2}) 
allows HD-MILP-Plan to 
obtain a tighter bound by exploring signigicantly less number of nodes. In the 
8-by-8 maze instance, we observe that HD-MILP-Plan with constraints (\ref{ref:S_0}-\ref{ref:S_2}) outperforms 
the base HD-MILP-Plan by 1700 and 3300 seconds with interval parameter $N=2,3$, 
respectively. In the 10-by-10 maze instance, we observe that HD-MILP-Plan with 
constraints (\ref{ref:S_0}-\ref{ref:S_2}) obtains a tighter bound compared to the 
base HD-MILP-Plan by 3750 
seconds and almost reaches the same bound by the time limit (i.e., 6000 seconds) 
with interval parameter $N=2,3$, respectively. 

The inspection of the top subfigures in Figures~\ref{fig:Strength_Nav_8} 
and~\ref{fig:Strength_Nav_10} shows that increasing the value of the interval parameter 
$N$ increases the computation time of Algorithm~\ref{rewpotalg}, but can also increase 
the search efficiency of the underlying B\&B solver through increasing its exploration 
and pruning capabilities, as demonstrated by the middle and bottom subfigures in 
Figures~\ref{fig:Strength_Nav_8} and~\ref{fig:Strength_Nav_10}. Overall from 
both instances, we conclude that HD-MILP-Plan with constraints (\ref{ref:S_0}-\ref{ref:S_2}) 
obtains a linear relaxation that is at least as good as the base HD-MILP-Plan by exploring 
significantly less number of nodes in the B\&B search tree.

\section{Related Work}

In this paper, we have focused on the important problem of improving the efficiency 
of B\&B solvers for optimal planning with learned NN transition models in continuous 
action and state spaces. Parallel to this work, 
planning and decision making in discrete action and state 
spaces~\cite{Lombardi2016,Say2018a,Say2018b},
verification of learned NNs~\cite{Katz2017,Ehlers2017,Huang2017,Narodytska2018}, 
robustness evaluation of learned NNs~\cite{Tjeng2019} and defenses to adversarial 
attacks for learned NNs~\cite{Kolter2017} have been studied with the focus of solving 
very similar decision making problems. For example, the verification problem solved by 
Reluplex~\cite{Katz2017}\footnote{Reluplex~\cite{Katz2017} is a SMT-based learned NN 
verification software.} is very similar to the planning problem solved by 
HD-MILP-Plan~\cite{Say2017} without the objective function and horizon $H=1$. 
Interestingly, the verification problem can also be modeled as an optimization 
problem~\cite{Bunel2017} and potentially benefit from the findings presented in 
this paper. For future work, we plan to explore how our findings in this work 
translate to solving other important tasks for learned neural networks.

\section{Conclusion}

In this paper, we have focused on the problem of improving the linear relaxation 
and the search efficiency of MILP models for decision making with learned NNs. 
In order to tacke this problem, we used bilevel programming to correctly model
the optimal reward potentials problem. We then introduced a novel finite-time 
constraint generation algorithm for computing the potential contribution of each 
hidden unit to the reward function of the planning problem. Given the precomputed 
values of the reward potentials, we have introduced constraints to 
tighten the bound on the reward function of the planning problem. Experimentally, 
we have shown that our constraint generation algorithm efficiently computes reward 
potentials for learned NNs, and the overhead computation is justified by the overall 
strengthening of the underlying MILP model as demonstrated on the task of planning 
over long horizons. With this paper, we have shown the \textit{potential} of bridging 
the gap between two seemingly distant literatures; the research on planning heuristics 
and decision making with learned NN models in continuous action and state spaces.

%
%
%
 \bibliographystyle{splncs04}
 \bibliography{bibfile}

\begin{thebibliography}{10}
\providecommand{\url}[1]{\texttt{#1}}
\providecommand{\urlprefix}{URL }
\providecommand{\doi}[1]{https://doi.org/#1}

\bibitem{Bard2000}
Bard, J.: Practical Bilevel Optimization: Algorithms And Applications. Springer
  US (09 2000). \doi{10.1007/978-1-4757-2836-1}

\bibitem{Bunel2017}
Bunel, R., Turkaslan, I., Torr, P.H., Kohli, P., Kumar, M.P.: A unified view of
  piecewise linear neural network verification  (2017)

\bibitem{Clarke2000}
Clarke, E., Grumberg, O., Jha, S., Lu, Y., Veith, H.: Counterexample-guided
  abstraction refinement. In: Emerson, E.A., Sistla, A.P. (eds.) Computer Aided
  Verification. pp. 154--169. Springer Berlin Heidelberg, Berlin, Heidelberg
  (2000)

\bibitem{Collobert2011}
Collobert, R., Weston, J., Bottou, L., Karlen, M., Kavukcuoglu, K., Kuksa, P.:
  Natural language processing (almost) from scratch. Journal of Machine
  Learning Research  \textbf{12},  2493--2537 (2011)

\bibitem{Deng2013}
Deng, L., Hinton, G.E., Kingsbury, B.: New types of deep neural network
  learning for speech recognition and related applications: an overview. In:
  IEEE International Conference on Acoustics, Speech and Signal Processing. pp.
  8599--8603 (2013)

\bibitem{Ehlers2017}
Ehlers, R.: Formal verification of piece-wise linear feed-forward neural
  networks. In: D'Souza, D., Narayan~Kumar, K. (eds.) Automated Technology for
  Verification and Analysis. pp. 269--286. Springer International Publishing,
  Cham (2017)

\bibitem{Huang2017}
Huang, X., Kwiatkowska, M., Wang, S., Wu, M.: Safety verification of deep
  neural networks. In: Majumdar, R., Kun{\v{c}}ak, V. (eds.) Computer Aided
  Verification. pp. 3--29. Springer International Publishing, Cham (2017)

\bibitem{IBM2019}
IBM: IBM ILOG CPLEX Optimization Studio CPLEX User's Manual (2019)

\bibitem{Katz2017}
Katz, G., Barrett, C., Dill, D., Julian, K., Kochenderfer, M.: Reluplex: An
  efficient smt solver for verifying deep neural networks. In: Twenty-Ninth
  International Conference on Computer Aided Verification. CAV (2017)

\bibitem{Kolter2017}
Kolter, Zico, W., Eric: Provable defenses against adversarial examples via the
  convex outer adversarial polytope. In: Thirty-First Conference on Neural
  Information Processing Systems (2017)

\bibitem{Krizhevsky2012}
Krizhevsky, A., Sutskever, I., Hinton, G.E.: Imagenet classification with deep
  convolutional neural networks. In: Twenty-Fifth Neural Information Processing
  Systems. pp. 1097--1105 (2012),
  \url{http://dl.acm.org/citation.cfm?id=2999134.2999257}

\bibitem{Lombardi2016}
Lombardi, M., Gualandi, S.: A lagrangian propagator for artificial neural
  networks in constraint programming. vol.~21, pp. 435--462 (Oct 2016).
  \doi{10.1007/s10601-015-9234-6},
  \url{https://doi.org/10.1007/s10601-015-9234-6}

\bibitem{Nair2010}
Nair, V., Hinton, G.E.: Rectified linear units improve restricted boltzmann
  machines. In: Twenty-Seventh International Conference on Machine Learning.
  pp. 807--814 (2010), \url{http://www.icml2010.org/papers/432.pdf}

\bibitem{Narodytska2018}
Narodytska, N., Kasiviswanathan, S., Ryzhyk, L., Sagiv, M., Walsh, T.:
  Verifying properties of binarized deep neural networks. In: Thirty-Second
  AAAI Conference on Artificial Intelligence. pp. 6615--6624 (2018)

\bibitem{Pommerening2015}
Pommerening, F., Helmert, M., R¨oger, G., Seipp, J.: From non-negative to
  general operator cost partitioning. In: Twenty-Ninth AAAI Conference on
  Artificial Intelligence. pp. 3335--3341 (2015)

\bibitem{Say2018b}
Say, B., Sanner, S.: Compact and efficient encodings for planning in factored
  state and action spaces with learned binarized neural network transition
  models (2018)

\bibitem{Say2018a}
Say, B., Sanner, S.: Planning in factored state and action spaces with learned
  binarized neural network transition models. In: Twenty-Seventh International
  Joint Conference on Artificial Intelligence. pp. 4815--4821 (2018).
  \doi{10.24963/ijcai.2018/669}, \url{https://doi.org/10.24963/ijcai.2018/669}

\bibitem{Say2017}
Say, B., Wu, G., Zhou, Y.Q., Sanner, S.: Nonlinear hybrid planning with deep
  net learned transition models and mixed-integer linear programming. In:
  Twenty-Sixth International Joint Conference on Artificial Intelligence. pp.
  750--756 (2017). \doi{10.24963/ijcai.2017/104},
  \url{https://doi.org/10.24963/ijcai.2017/104}

\bibitem{Jendrik2015}
Seipp, J., Pommerening, F., Helmert, M., R¨oger: New optimization functions
  for potential heuristics. In: Twenty-Fifth International Conference on
  Automated Planning and Scheduling. pp. 193--201 (2015)

\bibitem{Tjeng2019}
Tjeng, V., Xiao, K., Tedrake, R.: Evaluating robustness of neural networks with
  mixed integer programming. In: Seventh International Conference on Learning
  Representations (2019)

\end{thebibliography}

\end{document}